\documentclass[12pt]{article}

\usepackage{amsmath,amssymb,amstext,amsthm,mathtools,bbm,url}
\usepackage{tcolorbox,listings}

\newcommand{\eps}{\varepsilon}

\DeclareMathOperator{\opt}{OPT}

\DeclareMathOperator*{\argmin}{argmin}
\newcommand{\fcal}{\mathcal{F}}

\newtheorem{theorem}{Theorem}
\newtheorem{definition}[theorem]{Definition}

\newtheorem{lemma}[theorem]{Lemma}
\newtheorem{corollary}[theorem]{Corollary}
\newtheorem{remark}[theorem]{Remark}

\begin{document}
\title{Sample-Efficient Learning of Mixtures}
\author{Hassan Ashtiani\\
Department of Computer Science\\
University of Waterloo\\
mhzokaei@uwaterloo.ca
\and Shai Ben-David \\
Department of Computer Science\\
University of Waterloo\\
shai@uwaterloo.ca
\and Abbas Mehrabian \\
Simons Institute for Theory of Computing\\ 
University of California, Berkeley\\
abbasmehrabian@gmail.com}
\maketitle

\begin{abstract}%
We consider PAC learning of probability distributions (a.k.a.\ density estimation), where we are given an i.i.d.\ sample generated from an unknown target distribution, and want to output a distribution that is close to the target in total variation distance. Let $\mathcal F$ be an arbitrary class of probability distributions, and let $\mathcal{F}^k$ denote the class of $k$-mixtures of elements of $\mathcal F$. Assuming the existence of a method for learning $\mathcal F$ with sample complexity $m_{\mathcal{F}}(\epsilon)$, we provide a method for learning $\mathcal F^k$ with sample complexity $O({k\log k \cdot m_{\mathcal F}(\eps) }/{\epsilon^{2}})$. Our mixture learning algorithm has the property that, if the $\mathcal F$-learner is proper and agnostic, then the $\mathcal F^k$-learner would be proper and agnostic as well.

This general result enables us to improve the best known sample complexity upper bounds for a variety of important mixture classes. First, we show that the class of mixtures of $k$ axis-aligned Gaussians in $\mathbb{R}^d$ is PAC-learnable in the agnostic setting with $\widetilde{O}({kd}/{\epsilon ^ 4})$ samples,
which is tight in $k$ and $d$ up to logarithmic factors. Second, we show that the class of mixtures of $k$ Gaussians in $\mathbb{R}^d$ is PAC-learnable in the agnostic setting with sample complexity $\widetilde{O}({kd^2}/{\epsilon ^ 4})$, which improves the previous known bounds of $\widetilde{O}({k^3d^2}/{\epsilon ^ 4})$ and $\widetilde{O}(k^4d^4/\epsilon ^ 2)$ in its dependence on $k$ and $d$. Finally, we show that the class of mixtures of $k$ log-concave distributions over $\mathbb{R}^d$ is PAC-learnable using $\widetilde{O}(d^{(d+5)/2}\eps^{-(d+9)/2}k)$ samples.
\end{abstract}









\section{Introduction}
Learning distributions is a fundamental problem in statistics and computer science, and has numerous applications in machine learning and signal processing. 
The problem can be stated as:
{\begin{quote}{Given an i.i.d.\ sample generated from an unknown probability distribution $g$, find a distribution $\hat{g}$ that is close to $g$ in total variation distance.\footnote{Total variation distance is a prominent distance measure between distributions. For a discussion on this and other choices see~\cite[Chapter~5]{devroye_book}.
}
}\end{quote}}
This strong notion of learning is not possible in general using a finite number of samples. However, if we assume that the target distribution belongs to or can be approximated by a family of distributions, then there is hope to acquire algorithms with finite-sample guarantees. In this paper, we study the important family of mixture models within this framework.

Notice that we consider PAC learning of distributions (a.k.a.\ density estimation), which is different from parameter estimation.
In the parameter estimation problem, it is assumed that the target distribution belongs to some parametric class, and the goal is to learn/identify the parameters (see, e.g.,~\cite{dasgupta1999learning,Belkin,moitravaliant}).

As an example of our setting, assume that the target distribution is a Gaussian mixture with $k$ components in $\mathbb{R}^d$. Then, how many examples do we need to find a distribution that is $\epsilon$-close to the target? This \emph{sample complexity} question, as well as the corresponding \emph{computational complexity} question, has  received a lot of attention recently (see, e.g.~\cite{axis_aligned,onedimensional,spherical,robustestimation,gaussian_mixture,nearlylinear}).

In this paper, we consider a scenario in which we are given a method for learning a class of distributions (e.g., Gaussians). Then, we ask whether we can use it, as a black box, to come up with an algorithm for learning a mixture of such distributions (e.g., mixture of Gaussians). We will show that the answer to this question is affirmative.

We propose a generic method for learning mixture models. Roughly speaking, we show that by going from learning a single distribution from a class
to learning a mixture of $k$ distributions from the same class, the sample complexity is multiplied by a factor of at most ${(k\log^2 k)}/{\epsilon^2}$. This result is general, and yet it is surprisingly tight in many important cases.
In this paper, we assume that the algorithm knows the number of components $k$.

As a demonstration, we show that our method provides a better sample complexity for learning mixtures of Gaussians than the state of the art. In particular, for learning mixtures of $k$ Gaussians in $\mathbb{R}^d$, our method requires $\widetilde{O}({d^2k}/{\epsilon^4})$ samples, improving by a factor of $k^2$ over the $\widetilde{O}({d^2k^3}/{\epsilon^4})$ bound of~\cite{gaussian_mixture}. 
Furthermore, for the special case of mixtures of axis-aligned Gaussians, we provide an upper bound of  $\widetilde{O}({dk}/{\epsilon^{4}})$, which is the first optimal bound with respect to $k$ and $d$  up to logarithmic factors, and improves upon the 
$\widetilde{O}({dk^9}/{\epsilon^{4}})$ bound of~\cite{spherical}, which is only shown for the subclass of spherical Gaussians.

One merit of our approach is that it can be applied in the agnostic (a.k.a.\ robust) setting, where the target distribution does not necessarily belong to the mixture model of choice. 
To guarantee such a result, we need the black box to work in the agnostic setting. 
For example, an agnostic learning method for learning Gaussians can be lifted to work for Gaussian mixtures in the agnostic setting.

We would like to emphasize that our focus is on the information-theoretic aspects of learning rather than the computational ones; in particular, although our framework is algorithmic, its running time is exponential in the number of required samples. Proving sample complexity bounds is important in understanding the statistical nature of various classes of distributions (see, e.g., the recent work of \cite{logconcave}), and may pave the way for developing efficient algorithms.





\subsection{Our Results}

Let $\mathcal F$ be a class of probability distributions, and let $\mathcal{F}^k$ denote the class of $k$-mixtures of elements of $\mathcal F$.
In our main result, Theorem~\ref{thm:main}, assuming the existence of a method for learning $\mathcal F$ with sample complexity $m_{\mathcal{F}}(\epsilon)$,
we provide a method for learning $\mathcal F^k$ with sample complexity
$O({k\log^2 k \cdot m_{\mathcal F}(\eps) }/{\epsilon^{2}})$.
Our mixture learning algorithm has the property that, if the $\mathcal F$-learner is proper, then the $\mathcal F^k$-learner would be proper as well (i.e., the learner will always output a member of $\mathcal{F}^k$).
Furthermore, the algorithm works in the more general agnostic setting provided that the base learners are agnostic learners. 

We  provide several applications of our main result.
In Theorem~\ref{axisalignedupperbound}, we show that the class of mixtures of $k$ axis-aligned Gaussians in $\mathbb{R}^d$ 
is PAC-learnable in the agnostic setting with
sample complexity $O({kd\log^2 k}/{\epsilon ^ 4})$ (see Theorem~\ref{thm:lowerbound}).
This bound is tight in terms of $k$ and $d$ up to logarithmic factors.
In Theorem~\ref{gaussianupperbound}, we show that the class of mixtures of $k$ Gaussians in $\mathbb{R}^d$ 
is PAC-learnable in the agnostic setting with
sample complexity 
$O({kd^2\log^2 k}/{\epsilon ^ 4})$. Finally, in Theorem~\ref{logconcave}, we prove that the class of mixtures of $k$ log-concave distributions over $\mathbb{R}^d$ is PAC-learnable using $\widetilde{O}(d^{(d+5)/2}\eps^{-(d+9)/2}k)$  samples. To the best of our knowledge, this is the first upper bound on the sample complexity of learning this class.

\subsection{Related Work}
PAC learning of distributions was introduced by~\cite{Kearns}, we refer the reader to~\cite{Diakonikolas2016} for a recent survey. A closely related line of research in statistics (in which more emphasis is on sample complexity) is density estimation, for which the book by \cite{devroye_book} is an excellent resource.

One approach for studying the sample complexity of learning a class of distributions is bounding the VC-dimension of its associated Yatrocas class (see Definition~\ref{def_yatrocas}), and applying results such as Theorem~\ref{thm:distributionVC}\footnote{These VC-dimensions have mainly been studied for the purpose of proving generalization bounds for neural networks with sigmoid activation functions.}. In particular, the VC-dimension bound of \cite[Theorem~8.14]{AB99} -- which is based on the work of~\cite{Karpinski} -- implies a sample complexity upper bound of 
 $O((k^4d^2+k^3d^3)/\eps^2)$
for PAC learning mixtures of 
axis-aligned Gaussians,
and an upper bound of $O(k^4d^4/\eps^2)$ for PAC learning mixtures of general Gaussians (both results hold in the more general agnostic setting).

A sample complexity upper bound of $O(d^2 k^3 \log^2k / \eps^4)$ for  learning mixtures of Gaussians
in the realizable setting was proved in~\cite[Theorem~A.1]{gaussian_mixture} (the running time of their algorithm is not polynomial).
Our algorithm is motivated by theirs, but we have introduced several new ideas in the algorithm and in the analysis,
which has resulted in improving the sample complexity bound by a factor of $k^2$ and an algorithm that works in the more general agnostic setting.

For mixtures of spherical Gaussians, 
a polynomial time algorithm for the realizable setting with sample complexity
$O(dk^9\log^2(d)/\eps^4)$
was proposed in \cite[Theorem~11]{spherical}.
We improve  their sample complexity by a factor of $\widetilde{O}(k^8)$, and moreover our algorithm works in the agnostic setting, too.
In the special case of $d=1$, a non-proper agnostic polynomial time algorithm  with the optimal sample complexity of $\widetilde{O}(k/\eps^2)$ 
was given in~\cite{onedimensional}, and a proper agnostic algorithm with the same sample complexity and  better running time was given in~\cite{gmm_1d_proper}.

An important question, which we do not address in this paper, is finding polynomial time algorithms for learning distributions.
See~\cite{gaussian_mixture} for the state-of-the-art results on computational complexity of learning mixtures of Gaussians.
Another important setting is computational complexity in the agnostic learning, see, e.g.,~\cite{robustestimation} for some positive results.

A related line of research is parameter estimation for mixtures of Gaussians, see, e.g.,~\cite{dasgupta1999learning,Belkin,moitravaliant}, who gave polynomial time algorithms for this problem assuming certain separability conditions (these algorithms are polynomial in the dimension and the error tolerance but exponential in the number of components).
Recall that parameter estimation is a more difficult problem and any algorithm for parameter estimation requires some separability assumptions for the target Gaussians\footnote{E.g.,\ consider the case that $k=2$ and the two components are identical; then there is no way to learn their mixing weights.}, whereas for density estimation no such assumption is needed.

We finally remark that characterizing the sample complexity of learning a class of distributions in general is an open problem, even for the realizable (i.e., non-agnostic) case (see~\cite[Open Problem~15.1]{Diakonikolas2016}).

\section{The Formal Framework}
Generally speaking, a \emph{distribution learning method} is an algorithm that takes a 
sample of i.i.d.\ points from distribution $g$
as input, and outputs (a description) of a distribution $\hat{g}$ as an estimation for $g$.
Furthermore, we assume that $g$ belongs to or can be approximated by class $\mathcal{F}$ of distributions,
and we may require that $\hat{g}$ also belongs to this class (i.e., proper learning).

Let $f_1$ and $f_2$ be two probability distributions defined over the Borel $\sigma$-algebra $\mathcal{B}$. The total variation distance between $f_1$ and $f_2$ is defined as
\[\|f_1- f_2\|_{TV} = \sup_{B\in \mathcal{B}}|f_1(B) - f_2(B)| = \frac{1}{2}\|f_1 - f_2\|_1 \:,
\]
where 
\[\|f\|_1\coloneqq \int_{-\infty}^{+\infty} |f(x)|\mathrm{d} x\] 
is the $L_1$ norm of $f$.
In the following definitions, $\mathcal{F}$ is a
 class of probability distributions, and $g$ is a  distribution not necessarily in $\mathcal{F}$.
Denote the set $\{1,2,...,m\}$ by $[m]$.
All logarithms are in the natural base.
For a function $g$ and a class of distributions $\fcal$, we define
\[
\opt(\fcal, g) \coloneqq \inf_{f\in \mathcal{F}}\|f-g\|_1
\]

\begin{definition}[$\eps$-approximation, $(\eps, C)$-approximation]
A distribution $\hat{g}$ is  an \emph{$\eps$-approximation} for $g$ if $ \|\hat{g}- g\|_1 \leq \eps$.
A distribution $\hat{g}$ is an \emph{$(\epsilon, C)$-approximation} for $g$ with respect to $\mathcal{F}$ if 
\[ \|\hat{g}- g\|_1 \leq C\times \opt(\fcal, g) + \eps\]
\end{definition}

\begin{definition}[PAC-Learning Distributions, Realizable Setting]
A distribution learning method is called a \emph{(realizable) PAC-learner} for $\mathcal{F}$ with sample complexity $m_{\mathcal{F}}(\epsilon, \delta)$, if for all distribution $g\in\mathcal F$ and all $\epsilon, \delta >0$, given $\epsilon$, $\delta$, and a sample of size $m_{\mathcal{F}}(\epsilon, \delta)$, with probability at least $1-\delta$ outputs an $\epsilon$-approximation of $g$.
\end{definition}

\begin{definition}[PAC-Learning Distributions, Agnostic Setting]
For $C>0$, a distribution learning method is called a \emph{$C$-agnostic PAC-learner for $\mathcal{F}$} with sample complexity $m_{\mathcal{F}}^C(\epsilon, \delta)$, if for all distributions $g$ and all $\epsilon, \delta >0$, given $\epsilon$, $\delta$, and a sample of size $m_{\mathcal{F}}^C(\epsilon, \delta)$, with probability at least $1-\delta$ outputs an $(\epsilon, C)$-approximation of $g$.\footnote{Note that in some papers, only the case $C\leq1$ is called agnostic learning, and the case $C>1$ is called semi-agnostic learning.}
\end{definition}

Clearly, a $C$-agnostic PAC-learner (for any constant $C$)
is also a realizable PAC-learner, with the same error parameter $\eps$.
Conversely a realizable PAC-learner can be thought of an $\infty$-agnostic PAC-learner.


\section{Learning Mixture Models}

Let $\Delta_n$ denote the $n$-dimensional simplex:

\[
\Delta_n \:= \{ (w_1,\dots,w_n) : w_i\geq 0, \sum_{i=1}^k w_i=1\}
\]

\begin{definition}
Let $\mathcal{F}$ be a class of probability distributions. Then the class of \emph{$k$-mixtures of $\mathcal{F}$}, written $\mathcal{F}^k$, is defined as 
$$\mathcal{F}^k \coloneqq \left\{\sum_{i=1}^{k}w_{i}f_{i}: (w_1,\dots,w_k)\in \Delta_k ,
f_1,\dots,f_k\in\mathcal F
\right\}.$$
\end{definition}

Assume that we have a method to PAC-learn $\mathcal{F}$. Does this mean that we can PAC-learn $\mathcal{F}^k$? And if so, what is the sample complexity of this task?
Our main theorem gives an affirmative answer to the first question, and provides a bound for sample complexity of learning $\mathcal{F}^k$.

\begin{theorem}
\label{thm:main}
Assume that $\mathcal{F}$ has a $C$-agnostic PAC-learner 
with sample complexity $m_{\mathcal{F}}^C(\epsilon, \delta) = {\lambda(\mathcal F,\delta)}/{\epsilon^\alpha}$ for some $C>0$, $\alpha \geq 1$ and some function
 $\lambda(\mathcal F,\delta) = \Omega(\log(1/\delta))$. 
Then there exists a $3C$-agnostic PAC-learner for the class
$\mathcal{F}^k$  requiring $m_{\mathcal{F}^k}^{3C}(\epsilon, \delta) =$

\[ O\left(\frac{\lambda (\mathcal F, \frac{\delta}{3k}) k\log k}{\epsilon^{\alpha+2}}\right)
=
O\left(\frac{k\log k \cdot m_{\mathcal F}(\eps, \frac{\delta}{3k}) }{\epsilon^{2}}\right)
\] 
samples.
\end{theorem}

Since a realizable PAC-learner is an $\infty$-agnostic PAC-learner, we immediately obtain the following corollary.

\begin{corollary}
Assume that $\mathcal{F}$ has a realizable PAC-learner 
with sample complexity $m_{\mathcal{F}}(\epsilon, \delta) = {\lambda(\mathcal F,\delta)}/{\epsilon^\alpha}$ for some $\alpha \geq 1$ and some function
 $\lambda(\mathcal F,\delta) = \Omega(\log(1/\delta))$. 
Then there exists a realizable PAC-learner for the class
$\mathcal{F}^k$  requiring $m_{\mathcal{F}^k}(\epsilon, \delta) =$

\[ O\left(\frac{\lambda (\mathcal F, \frac{\delta}{3k}) k\log k}{\epsilon^{\alpha+2}}\right)
=
O\left(\frac{k\log k \cdot m_{\mathcal F}(\eps, \frac{\delta}{3k}) }{\epsilon^{2}}\right)
\] 
samples.
\end{corollary}

Some remarks:
\begin{enumerate}

\item
Our mixture learning algorithm has the property that, if the $\mathcal F$-learner is proper, then the $\mathcal F^k$-learner is proper as well.

\item
The computational complexity of the resulting algorithm is exponential in the number of required samples.

\item
The condition
$\lambda(\mathcal F,\delta) = \Omega(\log(1/\delta))$
is a technical condition that holds for all interesting classes $\mathcal F$.

\item
One may wonder about tightness of this theorem.
In Theorem~2 in \cite{spherical}, it is shown that 
if $\mathcal F$ is the class of spherical Gaussians,  we have
$m_{\mathcal{F}^k}^{O(1)}(\epsilon, \delta) = \Omega(k m_{\mathcal F}(\eps, \delta/k) )$, therefore, the factor of $k$ is necessary in general. However, it is not clear whether the additional factor of
$\log k /\eps^2$ in the theorem is tight.

\item
The constant 3 (in the $3C$-agnostic result) comes from 
\cite[Theorem~6.3]{devroye_book} (see Theorem~\ref{thm:candidates}), and it is not clear whether it is necessary. If we allow for randomized algorithms (which produce a random distribution whose expected distance to the target is bounded by $\eps$), then the constant can be improved to 2, see \cite[Theorem~22]{density_estimation_lineartime}. 
\end{enumerate}

In the rest of this section we prove
Theorem~\ref{thm:main}.
Let $g$ be the true data generating distribution, 
and let
\begin{equation}\label{def_rho}
g^* = \argmin_{f \in\mathcal{F}^k} \|g-f\|_1
\textnormal{ and }
\rho = \|g^*-g\|_1=\opt(\mathcal F^k,g)\:.
\end{equation}
Note that although $g^*\in\mathcal{F}^k$,  
$g$ itself is not necessarily in the form of a mixture. 
Since our algorithm works for mixtures, we would first like to write $g$ in the form of a mixture of $k$ distributions, such that they are on average close to being in $\fcal$.
This is done via the following lemma.

\begin{lemma}\label{lem:projection}
Suppose that $g$ is a probability density function with $\opt(\fcal^k,g)=\rho$.
Then we may write 
\(g = \sum_{i\in [k]} w_i {G_i},\)
such that $w\in\Delta_k$, 
each $G_i$ is a density, and
we have 
$\sum_{i\in[k]} w_i \opt(\mathcal F,G_i) = \rho$.
\end{lemma}


\begin{proof}
Let $f\in \fcal^k$ be such that $\|g-f \|_1\leq \rho$, and let $X=\{x: g(x) < f(x)\}$. 
Suppose $f = \sum_{i\in[k]} w_i f_i$, where $f_i\in \fcal$.
Define 
 \[
    G_i(x) = \left\{\begin{array}{lr}
        {f_i(x)g(x)}/{f(x)}, & \text{for } x\in X \\
        f_i(x)+\Delta_i(x), & \text{for } x\notin X\\
        \end{array}\right.
  \]
where 
\[
\Delta_i(x) = {(g(x)-f(x))\left(\int_X f_i(x)(f(x)-g(x)) \mathrm{d} x /f(x) \right)} \bigg/ \left({\int_X (f(x)-g(x)) \mathrm{d} x}\right).
\]
Observe that each $G_i$ is a density and that $g=\sum_{i\in[k]}w_i G_i$.
Finally, note that $f_i(x)>G_i(x)$ precisely on $X$. Thus,
\begin{align*} \rho=\|g-f \|_1 = 
2\int_X  (f(x) - \sum_i w_i G_i(&x)) 
=2\int_X \sum_i w_i (f_i(x)-G_i(x)) \\&=2\int_X \sum_i w_i |f_i(x)-G_i(x)| = \sum_i w_i \|f_i-G_i\|_1. \qedhere 
\end{align*}
\end{proof}

By Lemma~\ref{lem:projection}, we have
\(g = \sum_{i\in [k]} w_i {G_i},\)
where each $G_i$ is a probability distribution.
Let \(\rho_i \coloneqq \opt (\fcal, G_i)\),
and by the lemma we have  
\begin{equation}
\sum_{i \in[k]} w_i \rho_i = \rho.\label{sumrhoiisrho}.
\end{equation}
The idea now is to learn each of the $G_i$'s separately using the agnostic learner for $\fcal$.
We will view $g$ as a mixture of $k$ distributions
$G_1,G_2,\dots,G_k$.

For proving Theorem~\ref{thm:main}, we will use the following theorem on learning finite classes of distributions,
which immediately follows from \cite[Theorem~6.3]{devroye_book} and a standard Chernoff bound.
\begin{theorem}\label{thm:candidates}
Suppose we are given $M$ candidate distributions $f_1,\dots,f_M$ and we have access to i.i.d.\ samples from an unknown distribution $g$.
Then there exists an algorithm that given the $f_i$'s and  $\eps>0$, takes 
$\log (3M^2/\delta)/2\eps^2$ samples from $g$, and with probability  $\geq 1-\delta/3$ outputs an index $j\in[M]$ such that 
\[
\|f_j-g\|_1 \leq 3 \min_{i\in[M]} \|f_i-g\|_1 + 4\eps \:.
\]
\end{theorem}

We now describe an algorithm that with probability $\geq 1-\delta$ outputs a distribution with $L_1$ distance $13\eps + 3C\rho$ to $g$ (the error parameter is $13\eps$ instead of $\eps$ just for convenience of the proof; it is clear that this does not change the order of magnitude of sample complexity).
The algorithm, whose pseudocode is shown in Figure~\ref{fig:alg}, has two main steps.
In the first step we generate a set of candidate distributions, such that at least one of them is $(3\eps+\rho)$-close to $g$ in $L_1$ distance.
These candidates are of the form $\sum_{i=1}^{k} \widehat{w}_i \widehat{G}_i$, where the $\widehat{G}_i$'s are extracted from samples and are estimates for the real components $G_i$, 
and the $\widehat{w}_i$'s come from a fixed discretization of $\Delta_k$, and are estimates for the real mixing weights $w_i$.
In the second step, we use Theorem~\ref{thm:candidates} to obtain a distribution that is $(13\eps+3C\rho)$-close  to $g$.

\begin{figure}

\begin{tcolorbox}
Input: $k, \epsilon, \delta$ and an iid sample $S$\\
0. Let $\widehat{W}$ be an $(\epsilon/k)$-cover for $\Delta_k$ in $\ell_{\infty}$ distance.\\
1. $\mathcal C = \emptyset$ (set of candidate distributions)\\
2. For each $(\widehat{w}_1,\dots,\widehat{w}_k)\in \widehat{W}$ do:\\
\phantom{aa}3. For each possible partition of $S$ into \\
\phantom{aa aaaaaa}$A_1, A_2, ...,A_{k}$:\\
\phantom{aaaa}4. Provide $A_i$ to the $\mathcal F$-learner, and let $\widehat{G}_i$ \\\phantom{aaaa aaaaaa}be its output.\\
\phantom{aaaa}5. Add the candidate distribution \\ \phantom{aaaa aaaaaa}$ \sum_{i\in [k]} \widehat{w}_i \widehat{G}_i $ to $\mathcal C$.\\
6. Apply the algorithm for finite classes (Theorem~\ref{thm:candidates}) to $\mathcal C$ and output its result.
\end{tcolorbox}

\caption{Algorithm for learning the mixture class $\mathcal F^k$}
\label{fig:alg}
\end{figure}

We start with describing the first step.
We take
\begin{equation}
\label{s_def}
s=
\max \left \{
\frac
{2k\lambda(\mathcal F, \delta/3k)}
{ \eps^{\alpha} },
\frac{16 k \log(3k/\delta)}{\eps}
\right\}
\end{equation}
i.i.d.\ samples from $g$.
Let $S$ denote the set of generated points.
Note that $\lambda(\mathcal F,\delta)=\Omega(\log(1/\delta))$ implies 
\[s=O({k\lambda (\mathcal F, \delta/3k) }\times {\epsilon^{-\alpha}}).\]

Let $\widehat{W}$ be an $\eps/k$-cover for $\Delta_k$ in $\ell_{\infty}$ distance of cardinality 
$(k/\eps+1)^k$.
That is, for any $x\in \Delta_k$ there exists $w\in\widehat{W}$ such that $\|w-x\|_{\infty}\leq\eps/k$.
This can be obtained from a grid in $[0,1]^k$ of side length $\eps/k$, which is an $\eps/k$-cover for $[0,1]^k$, and projecting each of its points onto $\Delta_k$.

By an \emph{assignment}, we mean a function
$A:S\to [k]$.
The role of an assignment is to ``guess'' each sample point is coming from which component, by mapping them to a component index.
For each pair $(A,(\widehat{w}_1,\dots,\widehat{w}_k))$, where
$A$ is an assignment and $(\widehat{w}_1,\dots,\widehat{w}_k)\in \widehat{W}$,
we generate a candidate distribution as follows:
let $A^{-1}(i)\subseteq S$ be those sample points that are assigned to component $i$.
For each $i\in[k]$, we provide the set $A^{-1}(i)$ of samples to our $\mathcal F$-learner, and the learner  provides us with a distribution $\widehat{G}_i$.
We add the distribution $\sum_{i\in [k]} \widehat{w}_i \widehat{G}_i $
to the set of candidate distributions.

\begin{lemma}
\label{somecandidateisclose}
With probability $\geq 1-2\delta/3$, at least one of the generated candidate distributions is $(3\eps+C\rho)$-close to $g$.
\end{lemma}

Before proving the lemma, we show that it implies our main result, Theorem~\ref{thm:main}.
By the lemma, we obtain a set of candidates such that at least one of them is $(3\eps+C\rho)$-close to $g$ (with failure probability  $ \leq  2\delta/3$).
This step takes
$s=
O({k\lambda (\mathcal F, \delta/3k) }\times {\epsilon^{-\alpha}})$ many samples.
Then, we apply Theorem~\ref{thm:candidates} to output one of those candidates that is $(13\eps+3C\rho)$-close to $g$ (with failure probability $\leq \delta/3$), therefore using $\log(3M^2/\delta)/2\eps^2$ additional
samples. 
Note that the  number of generated candidate distributions is 
$M = k^s \times (1+k/\eps)^k$.
Hence, in the second step of our algorithm,
we take 
\begin{align*}
\log(3M^2/\delta)/2\eps^2
=
O\left(\frac{\lambda (\mathcal F, \delta/3k) k\log k}{\epsilon^{\alpha+2}}\right)
=
O\left(\frac{m_{\mathcal F}(\eps, \delta/3k) k\log k}{\epsilon^{2}}\right)
\end{align*}
additional samples.
The proof is completed noting the total failure probability is at most $\delta$ by the union bound.

We now prove Lemma~\ref{somecandidateisclose}.
We will use the following concentration inequality, which holds for any binomial random variable $X$ (see \cite[Theorem~4.5(2)]{Mitzenmacher}):
\begin{equation}
\label{chernoffhalf}
\Pr\{X < \mathbf{E} X/2\} \leq \exp(-\mathbf{E} X/8)\:.
\end{equation}
Say a component $i$ is \emph{negligible}
if 
\[ w_i \leq \frac{8 \log(3k/\delta)}{s}
\]
Let $L\subseteq[k]$ denote the set of negligible components.
Let $i$ be a non-negligible component.
Note that, the number of points coming from component $i$
is binomial with parameters 
$s$ and $w_i$ and thus has mean 
$ s w_i$,
so (\ref{chernoffhalf})
implies that, with probability at least $1-\delta/3k$, $S$ contains at least 
$w_is/2 $ points from $i$.
Since we have $k$ components in total, the union bound implies that,
with probability at least $1-\delta/3$,
uniformly for all  $i\notin L$, 
$S$ contains at least 
$w_is/2 $ points from component $i$.

Now consider the pair $(A,(\widehat{w}_1,\dots,\widehat{w}_k))$ such that $A$ assigns samples to their correct indices,
and has the property that $|\widehat{w}_i-w_i|\leq\eps/k$ for all $i\in[k]$.
We claim that the resulting candidate distribution is $(3\eps+C\rho)$-close to $g$.

Let $\widehat{G}_1,\dots,\widehat{G}_k$ be the distributions provided by the learner.
For each $i\in[k]$ define 
\[
\eps_i \coloneqq \left ( \frac
{2\lambda(\mathcal F, \delta/3k)}
{ w_is }
\right)^{1/\alpha}
\]
For any $i\notin L$,
since there exists at least 
$ w_is/2 $ samples for component $i$,
and since 
\[
 w_is/2 = \lambda(\mathcal F, \delta/3k) \eps_i^{-\alpha} = m_{\mathcal F} (\eps_i, \delta/3k)\:,
\]
we are guaranteed that
$\|\widehat{G}_i-G_i\|_1\leq C\rho_i + \eps_i$ with probability $1-\delta/3k$
(recall that each $G_i$ is $\rho_i$-close to the class $\fcal$).
Therefore,
$\|\widehat{G}_i-G_i\|_1\leq C\rho_i+\eps_i$ holds uniformly over all $i\notin L$, with probability $\geq 1-\delta/3$.
Note that since $\alpha\geq1$, the function $w_i^{1-1/\alpha}$ is concave in $w_i$, so by Jensen's inequality we have
\[
\sum_{i\in [k]} w_i^{1-1/\alpha}
\leq k \left ( (\sum_{i\in [k]} w_i / k)^{1-1/\alpha}\right)
= k^{1/\alpha}\:,
\]
hence
\begin{align*}
\sum_{i\notin L}
w_i \eps_i =
\left ( \frac
{2\lambda(\mathcal F, \delta/3k)}
{s}
\right)^{1/\alpha}
\sum_{i\notin L} w_i^{1-1/\alpha}
 \leq
\left ( \frac
{2k\lambda(\mathcal F, \delta/3k)}
{s }
\right)^{1/\alpha}.
\end{align*}

Also recall from (\ref{sumrhoiisrho}) that $\sum_{i \in[k]}w_i \rho_i \leq \rho$.
Proving the lemma is now a matter of careful applications of the triangle inequality:
\begin{align*}
\left\|  \sum_{i\in [k]} \widehat{w}_i \widehat{G}_i 
- g\right\|_1
&=
\left\| \sum_{i\in [k]} \widehat{w}_i \widehat{G}_i 
- 
\sum_{i\in [k]}  w_i G_i
\right\|_1\\
&
\leq
\left\|\sum_{i\in [k]} {w}_i (\widehat{G}_i 
-  G_i)
\right\|_1 
+
\left\|\sum_{i\in [k]} (\widehat{w}_i-w_i) \widehat{G}_i
\right\|_1
\\
& \leq
\left\|\sum_{i\in L}{w_i} (\widehat{G}_i 
-  G_i)
\right\|_1 
+
\left\|\sum_{i\notin L}{w_i} (\widehat{G}_i 
-  G_i)
\right\|_1 
 +
\sum_{i\in [k]} |\widehat{w}_i-w_i| \left\|\widehat{G}_i
\right\|_1
\\
& \leq
2\sum_{i\in L}{w_i}  
+
\sum_{i\notin L}{w_i} (\eps_i + C\rho_i)
+
\sum_{i\in [k]} \eps/k \times 1
\\
&\leq
2k \times \frac{8 \log(3k/\delta)}{s}
+
\left ( \frac
{2k\lambda(\mathcal F, \delta/3k)}
{ s }
\right)^{1/\alpha}
 + C \rho + \eps
\\
&\leq \eps+\eps+\eps + C \rho\:,
\end{align*}
where for the last inequality we used the definition of $s$ in (\ref{s_def}).
This completes the proof of  Lemma~\ref{somecandidateisclose}.

\section{Learning Mixtures of Gaussians}
Gaussian Mixture Models (GMMs) are probably the most widely studied mixture classes with numerous applications; yet, the sample complexity of learning this class is not fully understood, especially when the number of dimensions is large. In this section, we will show that our method for learning mixtures can improve the state of the art for learning GMMs in terms of sample complexity. In the following, $\mathcal{N}_d(\mu, \Sigma)$ denotes a Gaussian density function defined over $\mathbb{R}^d$, with mean $\mu$ and covariance matrix $\Sigma$.

\subsection{Mixture of Axis-Aligned Gaussians}
\label{sec:axisaligned}
A Gaussian is called \emph{axis-aligned} if its covariance matrix $\Sigma$ is diagonal. 
The class of axis-aligned Gaussian Mixtures is an  important special case of GMMs that is thoroughly studied in the literature (e.g.~\cite{axis_aligned}).

\begin{theorem}{\label{thm:AAGaussianUpperBound}}
Let $\mathcal F$ denote the class of $d$-dimensional axis-aligned Gaussians.
Then $\mathcal F$ is 3-agnostic PAC-learnable with 
$m_{\mathcal F}^3(\eps,\delta)=O((d+\log(1/\delta))/{\epsilon ^ 2})$.
\end{theorem}

We defer the proof of this result to the appendix. 
Combining this theorem with Theorem~\ref{thm:main} we obtain the following result:

\begin{theorem}
\label{axisalignedupperbound}
The class $\mathcal{F}^k$ of mixtures of $k$ axis-aligned Gaussians in $\mathbb{R}^d$ 
is 9-agnostic PAC-learnable with
sample complexity 
$m_{\mathcal F^k}^9(\epsilon,\delta) = O({kd\log k \log(k/\delta)}/{\epsilon ^ 4})$. Accordingly, it is also PAC-learnable in the realizable case with the same number of samples.
\end{theorem}

This theorem improves the upper bound of 
$O(dk^9\log^2(d/\delta)/\eps^4)$
proved in \cite[Theorem~11]{spherical} for spherical Gaussians in the realizable setting. 
Spherical Gaussians are special cases of axis-aligned Gaussians in which all eigenvalues of the covariance matrix are equal, i.e., $\Sigma$ is a multiple of the identity matrix. 
The following minimax lower bound (i.e., worst-case on all instances) on the sample complexity of learning mixtures of spherical Gaussians is proved in the same paper.

\begin{theorem}[Theorem~2 in \cite{spherical}]\label{thm:lowerbound}
The class $\mathcal{F}^k$ of mixtures of $k$ axis-aligned Gaussians in $\mathbb{R}^d$ 
in the realizable setting has 
$m_{\mathcal{F}^k}(\epsilon,1/2)=\Omega(dk/\epsilon^2)$.
\end{theorem}


Therefore, our upper bound of Theorem~\ref{axisalignedupperbound} is optimal in terms of dependence on $d$ and $k$ (up to logarithmic factors) for axis-aligned Gaussians.

\subsection{Mixture of General Gaussians}

For general Gaussians, we have the following result.

\begin{theorem}{\label{thm:GaussianUpperBound}}
Let $\mathcal F$ denote the class of $d$-dimensional Gaussians.
Then, $\mathcal F$ is 3-agnostic PAC-learnable with 
$m_{\mathcal F}^3(\eps,\delta)=O((d^2+\log(1/\delta))/{\epsilon ^ 2})$.
\end{theorem}

We defer the proof of this result to the appendix. 
Combining this theorem with Theorem~\ref{thm:main}, we obtain the following result:

\begin{theorem}
\label{gaussianupperbound}
The class $\mathcal{F}^k$ of mixtures of $k$ Gaussians in $\mathbb{R}^d$ 
is 9-agnostic PAC-learnable with
sample complexity 
$m_{\mathcal F^k}^9(\epsilon,\delta) = O({kd^2\log k \log(k/\delta)}/{\epsilon ^ 4})$. Accordingly, it is also PAC-learnable in the realizable case with the same number of samples.
\end{theorem}

This improves by a factor of $k^2$ the upper bound 
of $O(k^3d^2 \log k / \eps^4)$
in the realizable setting,
proved in~\cite[Theorem~A.1]{gaussian_mixture}. 

Note that Theorem \ref{thm:lowerbound} 
gives a lower bound of 
$\Omega(kd/{\epsilon ^ 2})$ for
$m_{\mathcal F^k}(\epsilon,\delta)$,
hence the dependence of Theorem~\ref{gaussianupperbound} on $k$ is optimal (up to  logarithmic factors). 
However, there is a factor of $d/\epsilon^2$  between the upper and lower bounds. 

\section{Mixtures of Log-Concave Distributions}

A probability density function over $\mathbb{R}^d$ is log-concave if its logarithm is a concave function. The following result about the sample complexity of learning log-concave distributions is the direct consequence of the recent work of \cite{logconcave}.

\begin{theorem} Let $\mathcal{F}$ be the class of distributions corresponding to the set of all log-concave pdfs over $\mathbb{R}^d$. Then $\mathcal{F}$ is $3$-agnostic PAC learnable using $m^{3}(\epsilon, \delta)=O((d/\eps)^{(d+5)/2}\log^2(1/\eps))$ samples.
\end{theorem}

Using Theorem~\ref{thm:main}, we come up with the first result about the sample complexity of learning mixtures of log-concave distributions.

\begin{theorem}\label{logconcave}
The class of mixtures of $k$ log-concave distributions over $\mathbb{R}^d$ is $9$-agnostic PAC-learnable using $\widetilde{O}(d^{(d+5)/2}\eps^{-(d+9)/2}k)$ samples.
\end{theorem}

\section{Conclusions}
We studied PAC learning of classes of distributions that are in the form of mixture models, and proposed a generic approach for learning such classes in the cases where we have access to a black box method for learning a single-component distribution. We showed that by going from one component to a mixture model with $k$ components, the sample complexity is multiplied by a factor of at most $(k\log^2 k)/\epsilon^2$.

Furthermore, as a corollary of this general result, we provided upper bounds for the sample complexity of learning GMMs and axis-aligned GMMs---$O({kd^2\log^2 k}/{\epsilon ^ 4})$ and $O({kd\log^2 k}/{\epsilon ^ 4})$ respectively. Both of these results improve upon the state of the art in terms of dependence on $k$ and $d$.

It is worthwhile to note that for the case of GMMs, the dependence of our bound is $1/ \epsilon^4$. Therefore, proving an upper bound of $kd^2/\epsilon^2$ remains open.

Also, note that our result can be readily applied to the general case of mixtures of the exponential family. Let $\mathcal{F}_d$ denote the $d$-parameter exponential family. Then the VC-dimension of the corresponding Yatrocas class (see Definition~\ref{def_yatrocas}) is $O(d)$ (see Theorem 8.1 in \cite{devroye_book}) and therefore by Theorem \ref{thm:distributionVC}, the sample complexity of PAC learning $\mathcal{F}_d$ is $O(d/\epsilon^2)$. Finally, applying Theorem \ref{thm:main} gives a sample complexity upper bound of $\widetilde{O}(k d/ \epsilon^4)$ for learning $\mathcal{F}_d^k$.

\paragraph{Addendum.}
After an earlier version of this work was presented in AAAI 2018 (and appeared on arXiv), we obtained improved results for learning mixtures of Gaussians: the class of mixtures of $k$ axis-aligned
Gaussians in $\mathbb{R}^d$ is agnostic PAC-learnable with sample
complexity $\widetilde{O}(kd/\eps^2)$,
and 
the class of mixtures of $k$ general
Gaussians in $\mathbb{R}^d$ is agnostic PAC-learnable with sample
complexity $\widetilde{O}(kd^2/\eps^2)$,
The proof uses novel techniques,
see~\cite{2017-abbas} for details.

\appendix

\section{Proofs of Theorems \ref{thm:AAGaussianUpperBound} and \ref{thm:GaussianUpperBound}}

We follow the general methodology of \cite{devroye_book} to prove upper bounds on the sample complexity of learning Gaussian distributions. The idea is to first connect distribution learning to the VC-dimension of a class of a related set system (called the Yatrocas class of the corresponding distribution family), and then provide upper bounds on VC-dimension of this system.
Our Theorem \ref{thm:distributionVC} gives an upper bound for the sample complexity of agnostic learning, given an upper bound for the VC-dimension of the Yatrocas class.
We remark that a variant of this result, without explicit dependence on the failure probability, is proved
implicitly in \cite{onedimensional} and 
also appears explicitly in \cite[Lemma~6]{logconcave}.


\begin{definition}[$\mathcal{A}$-Distance] Let $\mathcal{A}\subset 2^X$ be a class of subsets of domain $X$. Let $p$ and $q$ be two probability distributions over $X$. Then the \emph{$\mathcal{A}$-distance} between $p$ and $q$ is defined as
$$
\|p-q \|_{\mathcal{A}} \coloneqq \sup_{A\in \mathcal{A}} |p(A) - q(A)|
$$
\end{definition}


\begin{definition}[Empirical Distribution] Let $S = \{x_i\}_{i=1}^{m}$ be a sequence of members of $X$. The \emph{empirical distribution} corresponding to this sequence is defined by $\hat{p}_S(x) =  \sum_{i=1}^{m} \frac{\mathbbm{1} \{x = x_i \}}{m}$. 
\end{definition}

The following lemma is a well known refinement of the uniform convergence theorem, see, e.g.,~\cite[Theorem~4.9]{AB99}.

\begin{lemma}\label{lemma:vc}
Let $p$ be a probability distribution over $X$. 
Let $\mathcal{A} \subseteq 2 ^X$ and let $v$ be the VC-dimension of $\mathcal{A}$. Then, there exists universal positive constants $c_1,c_2,c_3$ such that
$$\mathbf{Pr}_{S\sim p^m} \{\|p-\hat{p}_S \|_{\mathcal{A}} \geq \eps \}\leq 
\exp(c_1 + c_2 v - c_3 m \eps^2)\:.$$
\end{lemma}

\begin{definition}[Yatrocas class] \label{def_yatrocas}
For a class $\fcal$ of functions from $X$ to $\mathbb{R}$,
their \emph{Yatrocas class} is the family of subsets of $X$ defined as
\[
\mathcal{Y}(\fcal) \coloneqq
\left \{
\{x\in X : f_1(x)\geq f_2(x)\}
\textnormal{ for some }
f_1,f_2\in\fcal
\right\}
\]
\end{definition}


Observe that if $f,g \in \fcal$ then
\(\|f - g \|_{TV} = \|f - g \|_{\mathcal{Y}(F)}\).

\begin{definition}[Empirical Yatrocas Minimizer]
Let $\mathcal{F}$ be a class of distributions over domain $X$. 
The \emph{empirical Yatrocas minimizer} is defined as
$L^{\mathcal{F}}:\cup_{m=1}^{\infty} X^m \to \mathcal{F}$ satisfying 
$$L^{\mathcal{F}}(S) =\arg\min_{q\in \mathcal{F}} \|q - \hat{p}_S \|_{\mathcal{Y}(\mathcal{F})}.$$
\end{definition}

\begin{theorem}[PAC Learning Families of Distributions]
\label{thm:distributionVC}
Let $\mathcal{F}$ be a class of probability distributions, and let $S\sim p^m$ be an i.i.d.\ sample of size $m$ generated from an arbitrary probability distribution $p$, which is not necessarily in $\mathcal{F}$. Then with probability at least $1 - \delta$ we have
$$\|p - L^{\mathcal{F}}(S)\|_{TV} \leq 3 \opt({\mathcal{F}},{p})  +   \alpha \sqrt{\frac{v + \log \frac{1}{\delta}}{m}}       $$
where $v$ is VC-dimension of $\mathcal{Y}(\mathcal{F})$, and $\opt({\mathcal{F}},{p}) = \inf_{q^* \in \mathcal{F}} \|q^* - p \|_{TV}$, and $\alpha$ is a universal constant.
In particular, in the realizable setting $p\in\mathcal F$, we have
$$\|p - L^{\mathcal{F}}(S)\|_{TV} \leq   \alpha \sqrt{\frac{v + \log \frac{1}{\delta}}{m}}       $$
\end{theorem}

\begin{remark}
The $L_1$ distance is precisely twice the total variation distance.
\end{remark}

\begin{proof}
Let $q^* = \arg\min_{q \in \mathcal{F}} \|p - q \|_{TV}$,
so $\|q^*-p\|_{\mathcal{Y}(\mathcal{F})}\leq\|q^*-p\|_{TV}= \opt(\mathcal F,p)$.
Since
$L^{\mathcal{F}}(S), q^*\in\fcal$ we have
$
\|L^{\mathcal{F}}(S) - q^* \|_{TV} =\|L^{\mathcal{F}}(S) - q^* \|_{\mathcal{Y}(\mathcal{F})}$.
By Lemma~\ref{lemma:vc},
with probability $\geq1-\delta$
we have 
$\|p - \hat{p}_S \|_{\mathcal{A}} \leq
\alpha \sqrt{(v+\log \frac{1}{\delta})/{m}}$ for some universal constant $\alpha$.
Also, since $L^{\mathcal{F}}(S)$ is the empirical minimizer of the $\mathcal{Y}(\mathcal{F})$-distance, we have
$\|L^{\mathcal{F}}(S) - \hat{p}_S \|_{\mathcal{Y}(\mathcal{F})}
\leq
\|q^* - \hat{p}_S \|_{\mathcal{Y}(\mathcal{F})}$.
The proof follows from these facts combined with multiple applications of the triangle inequality:
\begin{align*}
\| p - L^{\mathcal{F}}(S)\|_{TV}  & \leq \|L^{\mathcal{F}}(S) - q^* \|_{TV} + \|q^* - p \|_{TV}\\
&=\|L^{\mathcal{F}}(S) - q^* \|_{\mathcal{Y}(\mathcal{F})} + \opt({\mathcal{F}},{p}) \\
&\leq \|L^{\mathcal{F}}(S) - \hat{p}_S \|_{\mathcal{Y}(\mathcal{F})} + \|\hat{p}_S - q^* \|_{\mathcal{Y}(\mathcal{F})} + \opt({\mathcal{F}},{p}) \\
&\leq \|q^* - \hat{p}_S \|_{\mathcal{Y}(\mathcal{F})} + \left(\|\hat{p}_S - p \|_{\mathcal{A}} + \|p - q^* \|_{\mathcal{Y}(\mathcal{F})}\right) +  \opt({\mathcal{F}},{p}) \\&
\leq \left(\|q^* - p \|_{\mathcal{Y}(\mathcal{F})} + \|p - \hat{p}_S \|_{\mathcal{A}}\right) +  \|p - \hat{p}_S \|_{\mathcal{Y}(\mathcal{F})} + 2\opt({\mathcal{F}},{p}) \\
& \leq \|q^* - p \|_{TV} + 2\|p - \hat{p}_S \|_{\mathcal{Y}(\mathcal{F})} + 2\opt({\mathcal{F}},{p}) \\
& \leq 2\alpha\sqrt{\frac{v + \log \frac{1}{\delta}}{m}} + 3\opt({\mathcal{F}},{p})\:.\qedhere
\end{align*}
\end{proof}

Theorem \ref{thm:distributionVC} provides a tool for proving upper bounds on the sample complexity of distribution learning. To prove Theorems \ref{thm:GaussianUpperBound} and \ref{thm:AAGaussianUpperBound}, it remains to show upper bounds on the VC dimensions of the Yatrocas class of (axis-aligned) Gaussian pdfs.

For classes $\mathcal F$ and $\mathcal G$ of functions, let
\[\operatorname{NN}(\mathcal{G})\coloneqq \left\{\{x : f(x)\geq0\}\textnormal{ for some }f\in \mathcal{G} \right\}\]
and
\[\Delta \mathcal{F} \coloneqq \{f_1-f_2 : f_1, f_2 \in \mathcal{F} \}\:, \]
and notice that 
\[\mathcal{Y}(\fcal) = \operatorname{NN}(\Delta \fcal).\]
We upper bound the VC-dimension of
$\operatorname{NN}(\Delta \fcal)$ via the following well known result in statistical learning theory, see, e.g.,~\cite[Lemma~4.2]{devroye_book}.

\begin{theorem}[Dudley] Let $\mathcal{G}$ be an $n$-dimensional vector space of real-valued functions.
Then $VC(\operatorname{NN}(\mathcal{G})) \leq n$.
\end{theorem}

Now let $h$ be an indicator function
for an arbitrary element in $\operatorname{NN}(f_1-f_2)$,
where $f_1,f_2$ are pdfs of (axis-aligned) Gaussians.
Then $h$ is a $\{0,1\}$-valued function and we have:
\begin{align*}
h(x) & = \mathbbm{1}\{\mathcal{N}(\mu_1, \Sigma_1) > \mathcal{N}(\mu_2, \Sigma_2)\}\\
&  = \mathbbm{1}\left\{ \alpha_1 \exp(\frac{-1}{2}(x-\mu_1)^T\Sigma_1^{-1}(x-\mu_1)  ) >  \alpha_2 \exp(\frac{-1}{2}(x-\mu_2)^T\Sigma_2^{-1}(x-\mu_2)  ) \right\}\\
& = \mathbbm{1}\left\{ (x-\mu_1)^T\Sigma_1^{-1}(x-\mu_1)  -(x-\mu_2)^T\Sigma_2^{-1}(x-\mu_2) - \log \frac{\alpha_2}{\alpha_1} >0 \right\}\:.
\end{align*}
The inner expression is a quadratic form, and the linear dimension of all quadratic functions is $O(d^2)$. 
Furthermore, for axis-aligned Gaussians, $\Sigma_1$ and $\Sigma_2$ are diagonal, and therefore, the inner function lies in an $O(d)$-dimensional space of functions spanned by $\{1,x_1,\dots,x_d,x_1^2,\dots,x_d^2\}$. Hence, by Dudley's theorem, we have the required upper bound ($d$ or $d^2$) on the VC-dimension of the Yatrocas classes. 
Finally, Theorems \ref{thm:GaussianUpperBound} and \ref{thm:AAGaussianUpperBound} follow from applying Theorem \ref{thm:distributionVC} to the class of (axis-aligned) Gaussian distributions.

\noindent\textbf{Acknowledgements.}
{We would like to thank the reviewers of the ALT conference and also Yaoliang Yu for 
pointing out mistakes in earlier versions of this paper.
Abbas Mehrabian was supported by an NSERC Postdoctoral Fellowship and a Simons-Berkeley Research Fellowship.}

\end{document}